\documentclass[11pt]{article}
\pdfoutput=1

\usepackage{natbib} 
\usepackage[utf8]{inputenc} 
\usepackage[T1]{fontenc}    
\usepackage[colorlinks,
            linkcolor=red,
            anchorcolor=blue,
            citecolor=blue
            ]{hyperref}
\usepackage{url}            
\usepackage{booktabs}       
\usepackage{amsfonts}       
\usepackage{nicefrac}       
\usepackage{microtype}      
\usepackage{xcolor}         
\usepackage{spencercommands}
\usepackage{algorithm}
\usepackage{algorithmic}
\usepackage{enumerate}
\usepackage{enumitem}
\usepackage{setspace}
\usepackage[left=1in, right=1in, top=1in, bottom=1in]{geometry}
\allowdisplaybreaks

\title{\huge Self-training Converts Weak Learners to \\Strong Learners in Mixture Models}

\author
{
    Spencer Frei\thanks{Equal contribution } \thanks{Department of Statistics, University of California, Los Angeles, CA 90095, USA; e-mail: {\tt spencerfrei@ucla.edu}}
    ~~~and~~~
    Difan Zou$^*$\thanks{Department of Computer Science, University of California, Los Angeles, CA 90095, USA; e-mail: {\tt knowzou@ucla.edu}}
    ~~~and~~~
    Zixiang Chen$^*$\thanks{Department of Computer Science, University of California, Los Angeles, CA 90095, USA; e-mail: {\tt chenzx19@cs.ucla.edu}}
	~~~and~~~
	Quanquan Gu\thanks{Department of Computer Science, University of California, Los Angeles, CA 90095, USA; e-mail: {\tt qgu@cs.ucla.edu}}
}

\begin{document}
\date{}
\maketitle

\begin{abstract}
    We consider a binary classification problem when the data comes from a mixture of two rotationally symmetric distributions satisfying concentration and anti-concentration properties enjoyed by log-concave distributions among others. 
    We show that there exists a universal constant $C_{\mathrm{err}}>0$ such that if a pseudolabeler $\bbeta_{\mathrm{pl}}$ can achieve classification error at most $C_{\mathrm{err}}$, 
    then for any $\varepsilon>0$, an iterative self-training algorithm initialized at $\bbeta_0 := \bbeta_{\mathrm{pl}}$ using pseudolabels $\hat y = \mathrm{sgn}(\langle \bbeta_t, \xb\rangle)$ and using at most $\tilde O(d/\varepsilon^2)$ unlabeled examples suffices to learn the Bayes-optimal classifier up to $\varepsilon$ error, where $d$ is the ambient dimension.  That is, self-training converts weak learners to strong learners using only unlabeled examples.   We additionally show that by running gradient descent on the logistic loss one can obtain a pseudolabeler $\bbeta_{\mathrm{pl}}$ with classification error $C_{\mathrm{err}}$ using only $O(d)$ labeled examples (i.e., independent of $\varepsilon$).  Together our results imply that mixture models can be learned to within $\varepsilon$ of the Bayes-optimal accuracy using at most $O(d)$ labeled examples and $\tilde O(d/\varepsilon^2)$ unlabeled examples by way of a semi-supervised self-training algorithm.      
\end{abstract}


\section{Introduction}\label{sec:intro}
Current state-of-the-art methods for computer vision and natural language understanding have relied upon \textit{self-training} methods. 
These methods are generally unsupervised or semi-supervised learning approaches that take advantage of massive unlabeled datasets to improve performance on benchmark machine learning tasks~\citep{devlin2019bert,chen2020simclr}.  As human-annotated labeled data is expensive to collect, any approach which can reduce the number of labeled examples necessary for good performance is very desirable.  

One common approach in semi-supervised and self-supervised learning is the usage of a \textit{pseudolabeler}, which generates labels for unlabeled data $\xb$ by using the outputs of a classifier $\xb \mapsto \hat y := \sgn(f(\xb; \bbeta))$
where the pseudolabeler $f(\xb; \bbeta)$ has weights $\bbeta$ that have been pre-trained on labeled data (or a combination of labeled and unlabeled data). This approach has been remarkably successful in improving performance on image recognition tasks~\citep{pham2021metapseudolabels,rizve2021defensepseudolabel}, although there is very little theoretical understanding for why this method can improve performance or reduce the labeled sample complexity of the learning problem.

In this work, we provide algorithmic guarantees for the error of linear classifiers trained using only unlabeled samples using a standard self-training framework.   We assume the learner has access to an initial classifier $\bpl$
which could be generated in an arbitrary manner.  Given unlabeled examples $\{\xb_i\}_{i=1}^n$, initial classifier $\bbeta_0 :=\bpl$, at each time $t$ we generate pseudolabels $\hat y_i = \sgn(\sip{\bbeta_t}{\xb_i})$ that are then used in a standard gradient-based optimization of a weight-normalized loss of the form $\ell(\hat y_i \sip{\bbeta_t}{\xb_i} / \norm {\bbeta_t})$.
We assume the data is generated by a mixture model with two modes in the sense that, for labels $y\in \{\pm 1 \}$ and mean parameter $\bmu \in \R^d$, $\xb|y$ is a random variable with mean $y\bmu$ and the distribution of $\zb := \xb-y\bmu$ is unimodal, spherically symmetric, and satisfies some mild concentration and anti-concentration properties.  

Our main contributions are as follows.
\begin{enumerate}[leftmargin = *]
    \item[(1)] Provided the classification error of the initial pseudolabeler is smaller than some absolute constant $\cerr$, self-training with $\tilde O(d/\eps^2)$ unlabeled examples produces a classifier that has classification error at most $\eps$ larger than the Bayes-optimal error.
    \item[(2)] If the mixture model is sufficiently separated (i.e., $\norm \bmu \geq C_{\bmu}$ for some absolute constant $C_{\bmu}$), then in the \textit{supervised} learning setting, gradient descent on the logistic loss finds a classifier with classification error at most $\cerr$ using only $O(d)$ labeled examples---i.e., independent of $\eps$. 
    \item[(3)] Putting (1) and (2) together implies that in the semi-supervised setting, mixture models can be learned to within $\eps$ of the Bayes-optimal accuracy using $O(d)$ labeled examples and $\tilde O(d/\eps^2)$ unlabeled examples by using self-training with weight normalization.
\end{enumerate}

\paragraph{Organization of the paper.}  
We first discuss related work in Section~\ref{sec:related}.  We provide our main results on self-training with unlabeled examples in Section~\ref{sec:self.train.converts.weak.strong}.  In Section~\ref{sec:supervised}, we describe our results in the supervised setting and combine this with our results from Section~\ref{sec:self.train.converts.weak.strong} to get guarantees in the semi-supervised setting.  In Section~\ref{sec:proof}, we provide a proof sketch for our results on self-training.  We conclude in Section~\ref{sec:conclusion}, and leave detailed proofs for the appendices.  

\paragraph{Notation.}
We note here the notational conventions adopted in the paper.  We use bold letters to denote vectors.  We use $\norm {\xb}$ to denote the $\ell^2$ Euclidean norm of a vector $\xb$.  We say that $f(x) = O(g(x))$ if there exist universal constants $C, C'$ such that $f(x) \leq Cg(x)$ for $x\geq C'$; $f(x) = \Omega(g(x))$ if there exist $C,C'$ such that $f(x) \geq C g(x)$ for $x\geq C'$; and $f(x) = \Theta(g(x))$ if $f(x) = O(g(x))$ and $f(x) = \Omega(g(x))$.  We use $\tilde O, \tilde \Omega$, and $\tilde \Theta$ to additionally ignore logarithmic factors.  For a vector $\vb$, we denote by $\err(\vb) := \P_{(\xb,y)\sim \calD}(y \neq \sgn(\sip{\vb}{\xb})$, where the distribution $\calD$ will be understood from the context in which this term appears.  We use $\ind(A)$ to denote the indicator function of an event $A$, i.e. equal to one when the event $A$ occurs and zero otherwise.  The function $\sgn(z) = \ind(z > 0) - \ind(z < 0)$ is the sign function, equal to the sign of a real number with $\sgn(0)=0$.  We use the notation $a\wedge b$ to denote the minimum of $a$ and $b$, and the notation $a\vee b$ to denote the maximum of $a$ and $b$.  For a linear classifier $\xb \mapsto \sgn(\sip{\xb}{\bbeta})$ with parameter $\bbeta$, we will interchangeably refer to $\bbeta$ as the classifier or as the parameter.  We will likewise interchangeably refer to the parameters $\bpl$ defining a pseudolabeler, a pseudolabeler $\xb \mapsto \sip{\xb}{\bpl}$, and the classifier induced by a pseudolabeler $\xb \mapsto \sgn(\sip{\xb}{\bpl})$ itself, with the particular sense being clear in context. 

\section{Related Work}\label{sec:related}
Although the usage of the term `pseudolabel' dates back to as recently as 2013~\citep{lee2013pseudolabel}, the usage of self-supervised (unsupervised) methods to improve performance in supervised learning tasks has a long history in machine learning~\citep{schudder1965patternmachines,yarowsky1995unsupervised}.  It is only over the past few years that self-supervised `pre-training' methods have become a standard approach for improving performance in supervised learning tasks like image recognition and natural language understanding~\citep{devlin2019bert,chen2020simclr,pham2021metapseudolabels}.  Such methods are particularly appealing in the age of big data where in an increasing number of domains it is possible to collect massive unlabeled datasets.

From a theoretical perspective, much less is known about self-supervised and semi-supervised learning than in the supervised setting.  Early works by~\citet{castelli1995exponentialvaluelabeled,castelli1996relativevaluelabeled} looked at the relative value of labeled examples over unlabeled examples when the underlying marginal distribution of the features satisfies a parametric identifiability assumption.  A series of works have sought to clarify under what conditions semi-supervised learning can have provably better sample complexity or generalization performance in comparison with using solely supervised learning techniques~\citep{bendavid2008unlabeledprovablyhelp,singh2008unlabeleddatanowithelps,balcan2010discriminative,darnstdt2013unlabeleddatahelps,gopfert2019unlabeleddatalr}.  For surveys on early work in semi-supervised learning, we refer the reader to~\citet{zhu2009semisupintro} and~\citet{chapelle2010semisupbook}.  

More related to this work, a number of theorists have sought to better understand the mechanisms underlying the types of self-training algorithms used for deep neural networks.  This includes an analysis of contrastive learning~\citep{tosh2021contrastive}, consistency regularization~\citep{wei2021selftraining,cai2021labelpropagation}, robust self-training~\citep{raghunathan2020mitigatingtradeoffrobustness}, knowledge distillation~\citep{hsu2021generalizationdistillation} and masked feature prediction~\citep{lee2020predictingwhatyouknow}, to mention a few.  A number of works on the theory of self-training methods have focused on their applications in transfer learning and domain adaptation~\citep{kumar2020selftrainingradualdomain,chen2020selftrainingavoidspurious,xie2021innout}. 

A closely related paper is by \citet{oymak2020selftraininginsights}.  They considered the Gaussian mixture model setting and considered a self-training algorithm based on updating the estimate $\f 1 n \summ i n [y_i \xb_i] \approx \bmu$ for the mean of the mixture.  By replacing the labels $y_i$ with pseudolabels produced by some initial pseudolabeler, they are able to show that in the high dimensional limit, the predictors found by self-training are correlated with the Bayes-optimal predictor $\bmu$.  In contrast to our results, they did not provide a guarantee that their self-training algorithm converged to the Bayes-optimal predictor.  Additionally, the averaging operator they consider does not have analogues used in deep learning, which stands in contrast to the gradient-based training of the logistic loss we consider in this paper. 

\cite{kumar2020selftrainingradualdomain}, in a broad work on the usage of self-training methods for domain adaptation, worked on a similar problem to the one we consider in this paper.  They showed that in a Gaussian mixture model setting, assuming (1) iterative self-training solves an appropriate constrained nonconvex optimization problem and (2) access to infinite unlabeled data, then iterative self-training can yield the Bayes-optimal classifier provided it is initialized with a pseudolabeler with sufficiently small error.  By contrast, we directly show that the nonconvex optimization algorithm consisting of self-training with a finite set of unlabeled samples via weight-normalized gradient descent yields Bayes-optimal classifiers for a more general class of distributions. 

\cite{chen2020selftrainingavoidspurious} showed that for a mixture model where some coordinates are `spurious' and are distributed according to a (possibly anisotropic) Gaussian while the remaining coordinates satisfy mild distributional assumptions and fully determine the `signal' of the label $y$, self-training via \textit{projected} gradient descent learns to avoid the spurious features, provided the initial pseudolabeler does not depend much on the spurious features.  Under the additional assumption that all of the coordinates are Gaussian with only one coordinate determining the label, they are able to show self-training converges to an optimal classifier.  In comparison to our work, we have a more complete characterization of the sample complexity of the semi-supervised learning problem in that we show that a constant number of labeled examples is sufficient for learning pseudolabelers for which self-training learns optimal classifiers; we show convergence to the optimal classifier for more general distributions; and we consider the dynamics of self-training with \textit{weight-normalized} gradient descent, which are different from that of projected gradient descent.

\section{Self-training Converts Weak Learners to Strong Learners}\label{sec:self.train.converts.weak.strong}
In this section we show one of our key results, namely for data coming from an isotropic mixture model, there exists a universal constant $\cerr>0$ such that if an initial pseudolabeler $\bpl$ has classification error at most $\cerr$, then self-training using only unlabeled examples yields a classifier with classification error arbitrarily close to the Bayes-optimal error.  Before we begin, let us introduce some definitions which we will need to define the mixture model we consider.  Our first set of definitions are that of sub-exponential distributions and that of anti-concentration.
\begin{definition}[Sub-exponential distributions]\label{def:concentration}
We say $\calD_\xb$
is $K$\emph{-sub-exponential} if every $\xb\sim \calD_\xb$ is a sub-exponential random vector with sub-exponential norm at most $K$.  In particular, for any $\bar \vb$ with $\norm {\bar \vb}=1$, $\P_{\calD_x}(|\sip{\bar \vb}{\xb}| \geq t) \leq \exp(- t/K)$.  
\end{definition}

\begin{definition}\label{def:anticoncentration}
For $\bar \vb, \bar \vb'\in \R^d$, denote by $p_{\bar \vb}(\cdot)$ the marginal distribution of $\xb \sim \calD_\xb$
on the one dimensional subspace spanned by $\bar \vb$, and by $p_{\bar \vb,\bar \vb'}(\cdot)$ the marginal distribution on the subspace spanned by $\bar \vb$ and $\bar \vb'$.   We say the distribution satisfies \emph{$U$-anti-concentration} if there exists $U>0$ such that for any unit norm $\bar \vb\in \R^d$, $p_{\bar \vb}(t) \leq U$ for all $t\in \R$.  We say \emph{$(U',R)$-anti-anti-concentration} holds if there exists $U', R>0$ such that for any unit norm $\bar \vb, \bar \vb'\in \R^d$, it holds that $p_{\bar \vb, \bar \vb'}(\ab )\geq 1/U'$ for all $\ab \in \mathbb{R}^2$ satisfying $\norm{\ab}_2\leq R$.
\end{definition}
The sub-exponential definition is standard and satisfied by log-concave isotropic distributions among others.  Anti-concentration and anti-anti-concentration are fairly benign distributional assumptions, the former stating that the distribution cannot assign unbounded probability mass along one dimensional projections and the latter stating that the projection of the features onto low dimensional subspaces have probability density functions which assign at least a constant amount of mass near the origin.  A number of recent works have developed guarantees for learning halfspaces with noise under these distributional assumptions to avoid computational complexity lower bounds that exist without such assumptions~\citep{diakonikolas2019massart,diakonikolas2020nonconvex,frei2020halfspace,frei2021twolayerhalfspace,zoufrei2021adversarial}.  

We can now define the mixture distribution we consider in this work.

\begin{definition}\label{assumption:generative.model}
A joint distribution $\calD$ over $(\xb, y)\in \R^d \times \{\pm 1\}$ is defined as follows. Let $\bmu\in \R^d$, and $y = 1$ with probability $1/2$ and $y=-1$ with probability $1/2$. Then we generate $\xb|y \sim \zb + y\bmu$ where $\zb$ is an isotropic, rotationally symmetric\footnote{By isotropic we mean $\E[\xb]=0$ and $\E[\xb\xb^\top]=I$, and by rotationally symmetric we mean $\xb$ has the same distribution as $Q\xb$ for any orthogonal matrix $Q$.} and $K$-sub-exponential distribution satisfying $U$-anti-concentration and $(U', R)$-anti-anti-concentration.   Further assume $\zb$ is unimodal in the sense that its probability density function $p_\zb(z)$ is decreasing function of $\norm{z}_2$.  We call $(\xb, y)\sim \calD$ a \textit{mixture distribution with mean $\bmu$ and parameters $K, U, U',R$.}
\end{definition}

We note that log-concave isotropic distributions like the standard Gaussian are $K$-sub-exponential and satisfy $U$-anti-concentration and $(U', R)$-anti-anti-concentration with $K, U, U', R = \Theta(1)$~(see \citet[Fact 19]{diakonikolas2020massartstructured}).  Thus, our generative model is a natural generalization of the Gaussian mixture model that can accommodate a broader class of distributions.  We further note that the Bayes-optimal classifier for the mixture models we consider in this work is given by the linear classifier $\xb \mapsto \sgn(\sip{\bmu}{\xb})$.

\begin{fact}\label{fact:bayes.opt}
For mixture models satisfying Definition~\ref{assumption:generative.model}, the Bayes-optimal classifier is given by $\xb \mapsto \sgn(\sip{\bmu}{\xb})$. 
\end{fact}

A proof for Fact~\ref{fact:bayes.opt} is given in Appendix~\ref{appendix:bayes.opt}.  With the above in place, we can begin to describe the self-training algorithm we will use to amplify weak learners to strong learners using only unlabeled data.  We assume we have access to a pseudolabeler $\bpl$ which is able to achieve a sufficiently small, but constant, population-level classification error.  We then use a weight-normalized logistic regression method to train starting from the initial predictor $\bpl$ using only unlabeled examples.  Our results will rely upon loss functions that are well-behaved in the following sense.
\begin{definition}\label{assumption:loss}
We say a loss function $\ell$ is \textit{well-behaved} for some $C_\ell\geq 1$ if the loss $\ell(z)$ is 1-Lipschitz and decreasing on the interval $[0,\infty)$, and additionally $-\ell'(z) \geq \f 1{C_\ell} \exp(-z)$ for $z>0$.
\end{definition}
The exponential loss $\ell(z) = \exp(-z)$ and the logistic loss $\ell(z) = \log(1+\exp(-z))$ are well-behaved with $C_\ell=1$ and $C_\ell=2$ respectively.  Note that our analysis will \textit{not} require that the loss used is convex, merely that it is decreasing, Lipschitz, and that $-\ell'$ is bounded from below by a constant times the exponential loss.  Additionally note that we only specify the behavior of the loss on the interval $[0,\infty)$.  As we will see, this is because in the self-training algorithm we consider, the input to the loss function is always non-negative. 

We can now formally describe the self-training algorithm.  Let $\sigma>0$ be a parameter which we shall call the \textit{temperature}.  We assume we have access to $n = TB$ samples $\{\xit it\}_{i=1, \dots, B,\ t=0,\dots, T-1}$, which we partition into $T$ batches of size $B$.   With a well-behaved loss $\ell$, we define the (unsupervised) empirical risk
\begin{equation}\nonumber \label{eq:unsupervised.loss}
\hat {\lunsupt t} (\bbeta) := \f 1 B \summ i B \ell\l(\f 1 \sigma \cdot \sgn\l (\ip{\xit i t}{\bbeta} \r) \cdot \ip{\xit it}{\f{ \bbeta}{\norm{\bbeta}}} \r) = \f 1 B \summ i B \ell\l(\f 1 \sigma \l|\ip{\xit it}{\f{\bbeta}{\norm{\bbeta}}}\r| \r).
\end{equation}
That is, we use a typical weight-normalized logistic regression-type loss with pseudolabels given by $\hat y = \sgn(\sip{\bbeta}{\xb})$, with an additional factor given by the temperature $\sigma$.  We start with the predictor $\bbeta_0=\bpl / \norm{\bpl}$ and then use updates
\begin{align*}
\tilde \bbeta_{t+1} &= \bbeta_t - \eta \nabla \hat {\lunsupt t}(\bbeta_t),\\
\bbeta_{t+1} &= \tilde \bbeta_{t+1}/\norm{\tilde \bbeta_{t+1}}.
\end{align*}
Notice the usage of weight normalization in the definition of the unsupervised loss.  This can be viewed as a form of regularization for the learning algorithm, since if we do not normalize the weights it is possible that $\hat \lunsup(\bbeta)$ could be minimized by simply taking $\norm{\bbeta}\to \infty$.  The usage of a temperature term is common in self-training algorithms~\citep{hinton2015distillation,zou2019confidenceselftraining}, and has also previously been used for learning halfspaces with noise~\citep{diakonikolas2020nonconvex,zoufrei2021adversarial}.   We summarize the above into Algorithm~\ref{alg:selftraining}.

\begin{algorithm}[!t]
	\caption{Self-training with pseudolabels and weight normalization}
	\label{alg:selftraining}
	\begin{algorithmic}[1]
		\STATE \textbf{input:} 
		Training dataset $S = \{\xit it \}_{i=1,\dots,B,\ t=0,\dots, T-1}$, \\
		step size $\eta$, temperature $\sigma>0$, pseudolabeler $\bpl$
		\STATE $\bbeta_0 := \bpl / \norm{\bpl}$
 		\FOR {$t=0, \dots, T-1$}
 		\STATE Generate pseudolabels $\hat y_i^{(t)} = \sgn(\sip{\xit it}{\bbeta_t})$
		\STATE  $\tilde \bbeta_{t+1} = \bbeta_t - \frac{\eta}{B} \summ i B \nabla \ell(\f 1 \sigma \cdot  \hat y_i^{(t)} \cdot \sip{\xb_i^{(t)}}{ \bbeta_t / \norm{\bbeta_t}})$
		\STATE $\bbeta_{t+1} = \tilde \bbeta_{t+1}/\norm{\tilde \bbeta_{t+1}}$
		\ENDFOR 
		\STATE \textbf{output: $\bbeta_{T-1}$} 
	\end{algorithmic}
\end{algorithm}

Our main result is the following theorem.  We will present its proof in Section~\ref{sec:proof}. 
\begin{theorem}\label{thm:selftrain.unsupervised}
Suppose that $(\xb, y)\sim \calD$ follows a mixture distribution with mean $\bmu$ satisfying $\norm \bmu=\Theta(1)$ and parameters $K, U, U', R = \Theta(1)$.  Let $\ell$ be well-behaved for some $C_\ell\geq 1$, and assume the temperature satisfies $\sigma \geq R \vee \norm \bmu$.  Assume access to a pseudolabeler $\bpl$ which satisfies $\P_{(\xb, y)\sim \calD}\big(y\neq \sgn(\sip{\bpl}{\xb})\big) \leq \cerr$, where $\cerr = R^2 / (72 C_\ell U')$.   Let $\eps, \delta\in (0,1)$, and assume that
\[ B = \tilde \Omega\l(\eps^{-1}\r) , \quad T = \tilde \Omega \l( d\eps^{-1}   \r), \quad \eta = \tilde \Theta\l(d^{-1} \eps\r).\] 
\sloppy Then with probability at least $1-\delta$, by running Algorithm \ref{alg:selftraining} with step size $\eta$ and batch size $B$, the last iterate satisfies $\err(\bbeta_{T-1}) \leq \err(\bmu) + \eps$.   In particular, $T = \tilde O(d / \eps)$ iterations using at most $TB = \tilde O(d/\eps^2)$ unlabeled samples suffices to be within $\eps$ error of the Bayes-optimal classifier.
\end{theorem}

Theorem \ref{thm:selftrain.unsupervised} shows that provided we have a pseudolabeler which achieves a constant level of classification error, then by using only unlabeled examples, self-training with pseudolabels and weight normalization will amplify the pseudolabeler from a weak learner (achieving a constant level of accuracy) to a strong learner (achieving accuracy arbitrarily close to that of the best possible).  Note that for the mixture model, $\sgn(\sip{\bmu}{\cdot})$ is the Bayes-optimal classifier over the distribution (see Fact~\ref{fact:bayes.opt}), and if $\norm{\bmu}$ is small then the best error achievable might be quite large, as the region near the origin could have a large mass of samples that are just as likely to be from the $y=+1$ cluster and the $y=-1$ cluster (consider a mixture of two isotropic 2D Gaussians with means $(+1,0)$ and $(-1,0)$).  Thus in some settings it may not be possible for a pseudolabeler to have error smaller than $\cerr$.  However, we will see in the next section that provided $\norm \bmu$ is bounded below by a universal constant, we can ensure that a classifier trained by gradient descent using only $O(d)$ labeled examples has classification error at most $\cerr$.

\section{Semi-supervised Learning with $O(d)$ Labeled Examples via Self-training}\label{sec:supervised}
Theorem \ref{thm:selftrain.unsupervised} tells us that provided the self-training procedure (Algorithm \ref{alg:selftraining}) starts with a pseudolabeler that has classification error smaller than some absolute constant $\cerr$, self-training will boost this weak learner to a strong learner quickly.  In this section, we show that a standard logistic regression procedure produces a pseudolabeler that can achieve the desired constant accuracy by using only $O(d)$ samples---that is, a constant number of samples with respect to $\eps$.   The particular algorithm we consider is online SGD used to minimize the logistic loss $\ell(z) = \log(1+\exp(-z))$ defined over a linear classifier, and is given in Algorithm~\ref{alg:logistic}.  We use $O(\log(1/\delta))$ independent runs of online SGD to amplify a constant probability guarantee to a high probability guarantee. 


\begin{algorithm}[!t]
	\caption{Logistic regression with online stochastic gradient descent}
	\label{alg:logistic}
	\begin{algorithmic}[1]
		\STATE \textbf{input:} Failure probability $\delta \in (0,1)$,\\ 
		Training dataset $S = \{(\xit ti, y_t^{(i)}) \}$ for $t=0,\dots, T-1,\   i=1,\dots, 4\ceil{\log(1/\delta)}$, \\
		step size $\eta$.
		\STATE $\bbeta_0^{(i)} := 0.$
		\FOR {$i=1, \dots, 4\ceil{\log(1/\delta)}$}
 		\FOR {$t=0, \dots, T-1$}
 		\STATE $\bti {t+1}i = \bti ti - \eta \nabla \log(1 + \exp(-y_t^{(i)} \cdot \sip{\xit ti}{ \bti ti}))$
		\ENDFOR 
		\ENDFOR
		\STATE \textbf{output: $\{\bti ti\}_{t\in [T],\ i\in [\log(1/\delta)]}$} 
	\end{algorithmic}
\end{algorithm}

\begin{theorem}\label{thm:supervised.logistic.regression}
Suppose that $(\xb, y)\sim \calD$ follows a mixture distribution with mean $\bmu$ and parameters $K, U, U', R>0$.  Let $\cerr$ be the constant from Theorem~\ref{thm:selftrain.unsupervised} 
and assume $\norm \bmu \geq 3K \max(\log(8/\cerr), 22K)$.   By running Algorithm~\ref{alg:logistic} with $\eta = (\norm{\bmu}^2 + d)^{-1} \cerr /8$ and $T = 8 \eta^{-1} \cerr^{-1} \norm{\bmu}^2$ iterations, there exists $i\leq 4 \log(1/\delta)$ and $t<T$ such that with probability at least $1-\delta$, 
\[ \P(y \neq \sgn(\sip{\bti ti}{\xb})) \leq \cerr.\]
\end{theorem}
The proof of Theorem \ref{thm:supervised.logistic.regression} follows standard stochastic convex optimization arguments and can be found in Appendix~\ref{appendix:supervised}.  

Theorem \ref{thm:supervised.logistic.regression} implies that if we have access to $O\big(  (\norm{\bmu}^2 + d) \norm{\bmu}^2\big)$ labeled examples, where $O(\cdot)$ hides universal constants depending on $K$, $U$, $U'$, and $R$, we can learn a pseudolabeler $\bpl$ with classification error at most $\cerr$.  In particular, for $\norm \bmu = \Theta(1)$, using only $O(d)$ labeled examples suffices to learn a pseudolabler with error at most $\cerr$.  We can then use this pseudolabeler in Theorem \ref{thm:selftrain.unsupervised} with $O(d/\eps^2)$ \textit{unlabeled} examples to perform self-training and yield a classifier which achieves classification error at most $\eps$ larger than the best-possible error.  We collect these results into the following corollary.

\begin{corollary}\label{corollary:semisup}
Let $(\xb, y)\sim \calD$ be a mixture model with mean $\bmu$ and parameters $K,U, U', R = \Theta(1)$.  Assume $\norm \bmu = \Theta(1)$ satisfies
\[ \norm\bmu \geq 3K \max \l( \log( 144 U'/R^2), 22K\r),\]
Then for any $\eps, \delta \in(0,1)$, with probability at least $1-\delta$, using $O(d)$ labeled examples in Algorithm~\ref{alg:logistic} and $\tilde O(d/\eps^2)$ unlabeled examples in Algorithm~\ref{alg:selftraining} suffices to learn a predictor $\bbeta$ to within $\eps$ error of the Bayes-optimal classification error, where $O(\cdot)$ hides constants depending on $K$, $U$, $U'$, $R$, and $\log(1/\delta)$ only, and $\tilde O$ additionally suppresses logarithmic dependence on $\eps^{-1}$ and $d$. 
\end{corollary}

To the best of our knowledge, Corollary~\ref{corollary:semisup} is the first result to show that a semi-supervised self-training algorithm can learn an optimal classifier using only a constant number of labeled examples. 

On a related note, we want to acknowledge that for Gaussian mixture models, there exist purely unsupervised techniques (based on clustering methods) for which $\tilde O(d/\eps)$ unlabeled examples suffices to learn within $\eps$ of the \textit{clustering error} $\min\big(\P(y \neq \sgn(\sip{\bmu}{\xb})), 1 - \P(y\neq \sgn(\sip \bmu {\xb}))\big)$~\citep{li2017minimaxgaussianmixture}.  Thus, under more restrictive distributional assumptions and using algorithms designed for mixture models, it is possible to optimally learn a mixture model using only unlabeled examples.  We note this to emphasize that in this work we do not make the claim that self-training with pseudolabels is the optimal algorithm for learning mixture models.  Rather, our aim is to develop a better understanding of how self-training with pseudolabels can achieve good performance using few labeled examples.

\section{Proof of Main Results}\label{sec:proof}
In this section we provide a proof for Theorem~\ref{thm:selftrain.unsupervised}.  The key to our proof comes from deriving a lower bound that takes the form
\begin{equation}
    \sip{\bar\bmu}{-\nabla \hat {\lunsupt t}(\bbeta_t)} \geq C_0 \sin^2( \theta_t),\label{eq:key.inequality.sketch}
\end{equation}
where $\theta_t\in [0,\pi/2]$ is the angle between $\bbeta_t$ and $\bar \bmu$ and $C_0$ is some absolute constant.  To see the importance of such an inequality, let us look at the increments between the weights found using Algorithm~\ref{alg:selftraining} and those of the (normalized) ideal predictor $\bar \bmu := \bmu/\norm{\bmu}$.  Denote $\Delta_t^2 = \norm{\bbeta_t-\bar\bmu}^2$.  Let $\tilde \Delta_t^2 = \|\tilde\bbeta_t-\bar\bmu\|_2^2$.  Then,
\begin{align}\nonumber 
    \Delta_t^2 - \Delta_{t+1}^2 &\overset{(i)}\ge \Delta_t^2 - \tilde\Delta_{t+1}^2\notag\\\nonumber 
    &= 2 \eta \sip{\nabla \hat {\lunsupt t}(\bbeta_t)}{\bbeta_t-\bar\bmu} - \eta^2 \norm{ \nabla \hat {\lunsupt t}(\bbeta_t)}^2 \\\nonumber 
    &\overset{(ii)}= 2 \eta \sip{- \nabla \hat {\lunsupt t}(\bbeta_t)}{\bar\bmu} - \eta^2 \norm{\nabla \hat {\lunsupt t}(\bbeta_t)}^2 \\
    &\overset{(iii)}{\geq} 2\eta  \sip{- \nabla \hat {\lunsupt t}(\bbeta_t)}{\bar\bmu} - \eps.\label{eq:increment.sketch}
\end{align}
Inequalities $(i)$ and $(ii)$ follow from the fact that $\nabla \hat {\lunsupt t}(\bbeta_t)$ is orthogonal to $\bbeta_t$, as can be seen by the identity
\begin{align}\label{eq:gradient.formula}
\nabla \hat {\lunsupt t} (\bbeta_t) = \frac{1}{\sigma B\|\bbeta_t\|_2}\sum_{i=1}^B \ell' \bigg(\f 1 \sigma \frac{|\sip{\bbeta_t}{\xit it}| }{\|\bbeta_t\|_2}\bigg)\cdot\sgn(\sip{\bbeta_t}{\xit it})\cdot\bigg(I-\frac{\bbeta_t\bbeta_t^\top}{\|\bbeta_t\|_2^2}\bigg)\xit it.
\end{align}
In particular, $(i)$ follows from the identity  $\snorm{\tilde \bbeta_{t+1}}^2 = \snorm{\bbeta_t}^2 + \eta^2\snorm{\nabla \hat {\lunsupt t}(\bbeta_t)}^2>1$.   
Inequality $(iii)$ comes from taking $\eta$ sufficiently small.  Thus, if we have a lower bound like~\eqref{eq:key.inequality.sketch}, then \eqref{eq:increment.sketch} shows that whenever the angle $\theta_t$ between $\bbeta_t$ and $\bar \bmu$ is large, the distance between $\bbeta_t$ and $\bar \bmu$ will decrease.   Perhaps surprisingly, Lemma \ref{lemma:innerproduct_grad_both} below shows that one can guarantee this condition holds \textit{provided the predictor $\bbeta_t$ has classification error smaller than some absolute constant $\cerr$}.


\begin{lemma}
\label{lemma:innerproduct_grad_both}
Let $\calD$ be a mixture model with mean $\bmu$ and parameters $K, U,U', R>0$.  Let $\ell$ be well-behaved for some $C_\ell\geq 1$, and assume the temperature satisfies $\sigma \geq R \vee \norm \bmu$.  Suppose that $\norm{\bbeta_t}=1$ is an initial estimate.  Denote $\theta_t$ as the angle between $\bbeta_t$ and $\bmu$, and assume that $\theta_t \in [0,\pi/2]$.   Assume the classification error of $\bbeta_t$ satisfies
\begin{equation}\nonumber \label{eq:error.ub.req}
    \err_t := \P\big(y \neq \sgn(\sip{\bbeta_t}{\xb})\big) \leq \f{ R^2}{72 C_\ell U'} =: \cerr.
\end{equation}
Then we have
\[ \sip{ \bmu}{-\E \nabla \hat {\lunsupt t}(\bbeta_t)} \geq \f{ R^2 \norm \bmu^2}{36 \sigma C_\ell U'} \cdot \sin^2(\theta_t).\]
Moreover, there exists a universal constant $C_B>0$ such that for any $\eps, \delta\in (0,1)$,
\[ B \geq C_B \l( \f{ K C_\ell U'}{R^2}\r)^2 \eps^{-1} \log(2/\delta),\]
then with probability at least $1-\delta$,
\[ \sip{\bmu}{-\nabla \hat {\lunsupt t}(\bbeta_t)} \geq \f{ R^2 \norm{\bmu}^2}{72 \sigma C_\ell U'}  \sin^2\theta_t - \eps/2.\] 
\end{lemma}

The proof of the first part of Lemma \ref{lemma:innerproduct_grad_both} involves calculating integrals over the domain of $\xb$ and using the distributional properties of concentration and anti-anti-concentration.  To convert the guarantee for the population-level gradient into one for batches, we use concentration of sub-exponential random variables.  The proof is given in Appendix~\ref{appendix:selftrain}. 




With the above in hand, we can complete the proof of Theorem \ref{thm:selftrain.unsupervised}.
\begin{proof}[Proof of Theorem \ref{thm:selftrain.unsupervised}]
For notational simplicity, in the remainder of the proof let us denote
\begin{equation}\nonumber 
    C_g := \l( \f{R^2 \norm \bmu}{72 \sigma C_\ell U'}\r)^{-1},\quad  C_{d} := 2\norm \bmu^2 + 2d K^2 \log^2(dBT/\delta).
\end{equation}
Note that $C_g$ (the $g$ denoting gradient; see Lemma~\ref{lemma:innerproduct_grad_both}) is a universal constant independent of the dimension while $C_{d}$ depends on the dimension.  In the remainder of the proof we will use $\eta = \varepsilon/(16C_dC_g)$. 

Let $\bar \bmu := \bmu / \norm \bmu$.  Denote $\Delta_t^2 = \norm{\bbeta_t-\bar\bmu}^2$.  Let $\tilde \Delta_t^2 = \|\tilde\bbeta_t-\bar\bmu\|_2^2$.  
Using the same argument from~\eqref{eq:increment.sketch}, we have 
\begin{align}\label{eq:increment.base}
    \Delta_t^2 - \Delta_{t+1}^2 \geq 2 \eta \sip{-\nabla \hat {\lunsupt t}(\bbeta_t)}{\bar\bmu} - \eta^2 \norm{\nabla \hat {\lunsupt t}(\bbeta_t)}^2.
\end{align}
To control the gradient norm term, we use concentration.  Standard concentration of sub-exponential random variables gives (see Lemma \ref{lemma:concentration.subexp} for the full details) with probability at least $1-\delta$, for all $i\in [B]$ and $t\in [T]$,
\begin{equation}\label{eq:norm.ineq}
    \snorm{\xit it}^2 \leq 2 \norm{\bmu}^2 + 2 d K^2 \log^2(dBT/\delta) =:C_d.
\end{equation}
By Jensen's inequality,
\begin{equation}\label{eq:gradient.norm.bound}
\snorm{\nabla \hat{\lunsupt t}(\bbeta_t)}^2 \leq \f 1 {\sigma^2 B} \summ i B |\ell'(|\sip{\bbeta_t}{\xit it}| / \sigma )|^2 \snorm{\xit it}^2 \leq \f {C_d}{\sigma^2}.
\end{equation}
Substituting~\eqref{eq:gradient.norm.bound} into~\eqref{eq:increment.base}, we get
\begin{equation}\label{eq:increment.2}
    \Delta_t^2 - \Delta_{t+1}^2 \geq 2 \eta \l[ \sip{-\nabla \hat {\lunsupt t}(\bbeta_t)}{\bar \bmu} - \eta C_d/\sigma^2\r].
\end{equation}

For the first part of our proof, we claim that for all $t=0,1,\dots,$ it holds that $\Delta_t\leq \Delta_0$,  $\theta_t\in [0,\pi/2]$, and $\err(\bbeta_t)\leq \cerr$.  We show this result by induction.  For the base case, we have for any $\bbeta$ of unit norm,
 \begin{align}\nonumber 
     \P(y \neq \sgn(\sip{\bbeta}{\xb}) &= \P(\sip{\bbeta}{y\xb} < 0) \\\nonumber 
     &= \P(\sip{\bbeta}{y\xb-\bmu} < -\sip{\bbeta}{\bmu}) \\
     \label{eq:error.cos.identity}
     &= \P(\sip{\bbeta}{\zb} < - \norm \bmu\cos \theta),
 \end{align}
 where $\theta$ denotes the angle between $\bbeta$ and $\bmu$.  Since $\err(\bpl)\leq \cerr < 1/2$ and $\zb$ is mean zero, we must have $\theta_0\in [0,\pi/2]$.  Thus the base case $t=0$ holds.   Now assume the result holds for $t\in \N$ and consider the case $t+1$.  
  Since $\bbeta_t$ and $\bar \bmu$ are each of unit norm, we have the identity
 \begin{equation}\label{eq:delta.sintheta.identity}
     \Delta_t^2 = \norm{\bbeta_t-\bar \bmu}^2 = 2(1-\cos \theta_t) = 4 \sin^2 (\theta_t/2) \implies \Delta_t = 2 \sin(\theta_t/2).
 \end{equation}
 
 By the induction hypothesis, $\err(\bbeta_t)\leq \cerr$ and $\theta_t\in [0,/\pi/2]$.  We can thus use Lemma~\ref{lemma:innerproduct_grad_both} (with $\eps$ from the lemma statement replaced with $\eps/8C_g$) and \eqref{eq:increment.2} to get 
\begin{align}\nonumber 
    \Delta_t^2 - \Delta_{t+1}^2 &\geq 2 \eta \l[ \f 1 {C_g} \sin^2(\theta_t)  - \f{\eps}{16C_g}- \f{ \eta C_d}{\sigma^2}\r] \\
    &\geq 2 \eta \l[ \f 1 {4 C_g} \Delta_t^2 -   \f{\eps}{16C_g}- \f{ \eta C_d}{\sigma^2} \r].\label{eq:lines268269}
\end{align}
In the last line we have used~\eqref{eq:delta.sintheta.identity} and that $\theta_t\in [0,\pi/2]$.  
We therefore have
\begin{align*}
    \Delta_{t+1}^2 &\leq \left(1 - \frac{ \eta}{2 C_g} \right) \Delta_t^2 + \eta \eps / ( 8 C_g) + 2 \eta^2 C_d/\sigma^2 \\
    &\overset{(i)}\leq \left( 1 - \frac{ \eta}{2 C_g}\right) \Delta_0^2 + \eta \eps / ( 8 C_g) + 2 \eta^2 C_d/\sigma^2 \\
    &= \Delta_0^2 - \eta \left( \frac{ \Delta_0^2}{2C_g} - \f{ \eps}{ 8 C_g} - \f{ 2 \eta C_d}{\sigma^2 } \right).
\end{align*}
In $(i)$ we have used that $\eta = \varepsilon/(16C_dC_g \sigma^2)$ implies $1-\eta /2C_g >0$ and the inductive hypothesis that $\Delta_t^2\leq \Delta_0^2$.   Thus, we see that the choice of $\eta$ implies (where we assume $\Delta_0^2 > \varepsilon$ without loss of generality),
\[ \frac{\Delta_0^2}{2 C_g} - \f{ \eps}{8C_g} - \f{2 \eta C_d}{\sigma ^2 } = \frac{ \Delta_0^2}{2 C_g} -  \frac{\varepsilon}{4 C_g} > 0.\]
Hence $\Delta_{t+1}^2 \leq \Delta_0^2$.  Using \eqref{eq:error.cos.identity} and~\eqref{eq:delta.sintheta.identity} and the induction hypothesis, this implies $\err(\bbeta_{t+1})\leq \cerr$ and $\theta_{t+1}\in [0,\pi/2]$.   This completes the induction and hence we have that for all $t$, $\err(\bbeta_t)\leq \cerr$ and $\theta_t\in [0,\pi/2]$ holds so that we may apply Lemma~\ref{lemma:innerproduct_grad_both} for every $t$. In particular, for every $t$, \eqref{eq:lines268269} holds.  We re-arrange~\eqref{eq:lines268269} to get for any $T\in \N$,
\begin{equation*}
    \Delta_T^2 \leq (1 - \eta /2C_g) \Delta_{T-1}^2 + \eta \eps / (8 C_g) +  2 C_d \eta^2/ \sigma^2.
\end{equation*}
One can verify (see Lemma~\ref{lemma:recursion.calculation} for the detailed calculation) that for $\eta = \eps/(16 C_d C_g \sigma^2)$ and provided the number of iterations satisfies $T \geq 32 C_d C_g^2 \sigma^2 \eps^{-1} \log(32 C_d C_g^2 \sigma^2 \eps^{-1})$, this implies
\begin{align*}
    4 \sin^2(\theta_T/2) = \Delta_T^2 \leq \eps.
\end{align*}
To convert the guarantee for the angle between $\bbeta_t$ and $\bmu$ into one on the gap of the classification error between $\bbeta_t$ and $\bar \bmu$, we use \eqref{eq:error.cos.identity} to write
\begin{align} \nonumber 
    \err(\bbeta_t) - \err(\bar \bmu) &= \P(\sip{\bbeta_t}{\zb} < - \norm \bmu \cos \theta_t) - \P( \sip{\bar \bmu}{\zb} < - \norm \bmu) \\ \nonumber 
    &\overset{(i)}= \P(\sip{\vb}{\zb} \in [- \norm \bmu, - \norm \bmu \cos \theta_t]) \\ \nonumber 
    &\overset{(ii)}\leq U \norm \bmu [1-\cos \theta_t] \\ \nonumber 
    &\overset{(iii)}\leq U \norm \bmu \sin^2\theta_t.
\end{align}
In $(i)$ we use that $\zb$ is rotationally invariant and that $\norm {\bbeta_t}=\norm{\bar \bmu}=1$ so that $\sip{\bbeta_t}{\zb}$ and $\sip{\bar \bmu}{\zb}$ have the same distribution as $\sip{\vb}{\zb}$ for a vector $\vb$ satisfying $\norm {\vb}=1$.  In $(ii)$ we have used the definition of $U$-anti-concentration. In $(iii)$ we have used the inequality $1-\cos\theta_t = 1-\sqrt{1-\sin^2\theta_t}\le \sin^2\theta_t$.  Since $\sin^2 \theta_{T} \leq 4\sin^2 (\theta_T/2) \leq \epsilon$, by rescaling $\eps$ to $\eps / (U \norm \bmu)$, we get the desired result. 

\end{proof}

\section{Conclusion}\label{sec:conclusion}
In this work we theoretically analyzed an increasingly popular semi-supervised learning method: self-training with pseudolabels via gradient based optimization of the cross-entropy loss following supervised learning with a limited number of samples.  We considered the setting of general mixture models satisfying benign concentration and anti-concentration properties.  We showed that provided the initial pseudolabeler has classification error smaller than some absolute constant $\cerr$, using $\tilde O(d/\eps^2)$ unlabeled samples suffices for the self-training procedure to get within $\eps$ classification error of the Bayes-optimal classifier for the distribution.  By showing that the standard gradient descent algorithm can learn a pseudolabeler with classification error at most $\cerr$ using only $O(d)$ labeled examples, our results provide the first proof that a constant (with respect to $\eps$) number of labeled examples suffices for optimal performance in a semi-supervised self-training algorithm. 

For future research, we are interested in developing better sample complexity guarantees for self-training with unlabeled examples.  We are additionally interested in understanding if some of the methods we have developed can translate to settings where the optimal classifier is nonlinear.  We expect this analysis to require novel non-convex optimization analyses.  Our hope is that such settings will allow for better insight into the usage of self-training in neural networks.


\appendix 
\section{Proofs from Section \ref{sec:self.train.converts.weak.strong}}\label{appendix:selftrain}

\subsection{Proof of Lemma \ref{lemma:innerproduct_grad_both}: expected value} 
In this section we prove the first part of Lemma~\ref{lemma:innerproduct_grad_both}, involving the lower bound given for $\sip{\bmu}{-\E \nabla \hat{\lunsupt t}(\bbeta_t)}$.  Our proof relies upon similar ideas used by~\citet{diakonikolas2020nonconvex} and~\citet{zoufrei2021adversarial} for learning halfspaces with agnostic noise.  At a high level, the noise in the halfspace setting considered by~\citep{diakonikolas2020nonconvex,zoufrei2021adversarial} corresponds to the error made by the pseudolabeler in our setting.  
\begin{lemma}[Lemma \ref{lemma:innerproduct_grad_both}, expected value]\label{lemma:innerproduct_grad_mean_expectation}
Let $(\xb, y)\sim \calD$ be a mixture model with mean $\bmu$ and parameters $K, U, U', R>0$.  Let $\ell$ be well-behaved for some $C_\ell\geq 1$, and assume the temperature satisfies $\sigma \geq R \vee \norm \bmu$.  Suppose that $\norm{\bbeta_t}=1$ is an initial estimate.  Denote $\theta_t$ as the angle between $\bbeta_t$ and $\bmu$, and assume that $\theta_t \in [0,\pi/2]$.   Assume the classification error of $\bbeta_t$ satisfies
\begin{equation}\label{eq:error.ub.req.restated} \nonumber 
    \err_t := \P(y \neq \sgn(\sip{\bbeta_t}{\xb})) \leq \f{ R^2}{72 C_\ell U'} =: \cerr.
\end{equation}
Then
\[ \sip{\bar \bmu}{-\E \nabla \hat {\lunsupt t}(\bbeta_t)} \geq \f{ R^2 \norm \bmu^2}{36 \sigma C_\ell U'} \cdot \sin^2(\theta_t).\]
\end{lemma}

\begin{proof}[Proof of Lemma \ref{lemma:innerproduct_grad_both}]
Since $\norm {\bbeta_t}=1$, using the gradient formula~\eqref{eq:gradient.formula},
\begin{align*}
-\EE[\nabla \hat {\lunsupt t}(\bbeta_t)] = \EE_{\xb \sim \calD_x}\big[-\ell' \big(|\sip{\bbeta_t}{\xb }|\big)\cdot\sgn(\sip{\bbeta_t}{\xb})\cdot\big(I-\bbeta_t\bbeta_t^\top\big)\xb \big].
\end{align*}
Denote
\begin{equation}\nonumber 
    \bar \bmu := \f{\bmu}{\norm{\bmu}},\quad \tilde \bmu_t := (I-\bbeta_t\bbeta_t^\top) \bar \bmu.
\end{equation}
We can then write
\begin{align*}
\sip{\bmu}{-\EE\nabla  \hat {\lunsupt t}(\bbeta_t)} = \norm{\bmu} \EE\big[-\ell' \big(|\sip{\bbeta_t}{\xb}|\big)\cdot\sgn(\sip{\bbeta_t}{\xb})\cdot\tilde \bmu^\top \xb\big].
\end{align*}
Define the event
\begin{equation} \nonumber 
    S_t := \{y = \sgn(\sip{\bbeta_t}{\xb}) \}.
\end{equation}
Then we can calculate
\begin{align}\nonumber 
&\E[ -\ell' \big(|\sip{\bbeta_t}{\xb}| / \sigma\big)\cdot\sgn(\sip{\bbeta_t}{\xb})\cdot\tilde \bmu ^\top \xb] \\
&= \E[ -\ell' \big(|\sip{\bbeta_t}{y\xb}| / \sigma\big)\tilde \bmu^\top (y\xb)\ind(S_t)] + \E[\ell' \big(|\sip{\bbeta_t}{y\xb}| / \sigma \big)\tilde \bmu ^\top (y\xb)\ind( S_t^c)]\notag\\
& = \E [-\ell' \big(|\sip{\bbeta_t}{y\xb}| / \sigma \big)\tilde \bmu^\top (y\xb)] + 2 \E [\ell' \big(|\sip{\bbeta_t}{y\xb}| / \sigma\big)\tilde \bmu^\top (y\xb)\ind( S_t^c)].\label{eq:l'.decomp}
\end{align}
In the remainder of the proof, we will derive a lower bound on the first quantity and an upper bound on the absolute value of the second quantity.

We proceed with the lower bound for the first term as follows.  Since the quantity depends only on the projection of $\zb$ onto the space spanned by $\bbeta_t$ and $\bar \bmu$, we work in this two dimensional space.  Since $\zb$ is rotationally invariant, we can rotate the coordinate system so that $\bbeta_t=e_2$, $\bar \bmu=(\sin \theta_t, \cos \theta_t)$, and $\tilde \bmu = (\sin \theta_t, 0)$, where $\theta_t$ is the angle between $\bbeta_t$ and $\bar \bmu$.  Denote $p_{\bbeta_t, \bmu}(\cdot, \cdot): \R^2 \to [0,\infty)$ as the probability density function of the projection of $\zb$ onto the 2D subspace spanned by $\bbeta_t$ and $\bmu$.  Then, using that $y\xb$ has the same distribution as $\zb + \bmu$,
\begin{align} \nonumber 
&\EE[-\ell' \big(|\sip{\bbeta_t}{y\xb}| / \sigma \big)\tilde \bmu^\top (y\xb)]\\  \nonumber 
& = \int_{-\infty}^\infty \int_{-\infty}^\infty -\ell'(|u_2 + \|\bmu\|\cos \theta_t| / \sigma ) \cdot \sin \theta_t \cdot (u_1+\|\bmu\|\sin \theta_t) \cdot p_{\bbeta_t, \bar \bmu}(u_1, u_2)  \dx u_1 \dx u_2 \\ \nonumber 
&=  \int_{-\infty}^\infty  -\ell'(|u_2 + \|\bmu\|\cos \theta_t| / \sigma ) \cdot \sin \theta_t \cdot \bigg[\int_{-\infty}^\infty(u_1+\|\bmu\|\sin \theta_t )\cdot p_{\bbeta_t, \bar \bmu}(u_1, u_2)  \dx u_1 \bigg]\dx u_2 \\
&\overset{(i)}=\norm \bmu \sin^2 \theta_t \int_{-\infty}^\infty \int_{-\infty}^\infty  -\ell'(|u_2 + \|\bmu\|\cos \theta_t| / \sigma )p_{\bbeta_t, \bar \bmu}(u_1, u_2) \dx u_1  \dx u_2.\label{eq:temperature.lowerbound.init}
\end{align}
In $(i)$ we use the fact that $\zb$ isotropic implies that the projection of $\zb$ onto one dimensional subspaces are mean zero, and thus for all $u_{2}$ we have $\int_{-\infty}^\infty u_1\cdot p_{\bbeta_t, \bar \bmu}(u_1, u_2)  \dx u_1 = 0$.  We now calculate a lower bound on the remaining quantity.   Since $\zb$ is rotationally invariant, we know that $p_{\bbeta_t,\bar \bmu}(u_1, u_2)$ depends only on the distance from the origin $\sqrt{u_1^2+u_2^2}$.  Thus, we can convert to polar coordinates and using an abuse of notation write $p_{\bbeta_t, \bar \bmu}(r)$ to emphasize that the density only depends on the distance $r$ from the origin in polar coordinates.  Continuing, this means we can write

\begin{align} \nonumber
&\int_{-\infty}^\infty  -\ell'(|u_2 + \|\bmu\|\cos \theta_t| / \sigma )p_{\bbeta_t, \bar \bmu}(u_1, u_2) \dx u_1  \dx u_2.
 \\  \nonumber 
 &=\int_{r=0}^\infty r p_{\bbeta_t, \bar \bmu}(r) \int_{\phi=-\pi}^\pi -\ell'(|r\cos\phi+ \norm \bmu \cos\theta_t | / \sigma) \dx r \dx \phi\\ \nonumber 
    &\overset{(i)}{\geq} \int_{r=0}^\infty   r p_{\bbeta_t, \bar \bmu}(r)  \int_{\phi=0}^{\pi/2}-\ell'(|r\cos\phi+ \norm \bmu \cos\theta_t | / \sigma ) \cdot \sin \phi \cdot \dx r \dx \phi\\ \nonumber 
    &\overset{(ii)}{\geq} \int_{r=0}^\infty r p_{\bbeta_t, \bar \bmu}(r)  \int_{\phi=0}^{\pi/2}\f 1 {C_\ell} \exp (- |r\cos\phi+ \norm \bmu \cos\theta_t| / \sigma ) \cdot \sin \phi \cdot \dx r \dx \phi\\ \nonumber 
    &\overset{(iii)}= \f{  \exp(- \f {\norm \bmu}\sigma \cos \theta_t )}{C_\ell} \int_{r=0}^\infty r p_{\bbeta_t, \bar \bmu}(r)  \int_{\phi=0}^{\pi/2} \exp (-r \cos (\phi) / \sigma ) \cdot \sin \phi \cdot \dx r \dx \phi\\ \nonumber 
    &\overset{(iv)}= \f{ \sigma  \exp(- \f{\norm \bmu}\sigma  \cos \theta_t)}{C_\ell} \int_{r=0}^\infty  p_{\bbeta_t, \bar \bmu}(r)  (1 - \exp(-r/\sigma)) \cdot \dx r  \\ \nonumber 
    &\overset{(v)}{\geq} \f {\sigma \exp(-\f {\norm \bmu}\sigma \cos \theta_t))}{ C_\ell U'} \int_0^R [1-\exp(-r/\sigma )] \dx r \\  \nonumber 
    &\overset{(vi)}\geq \f{ \sigma  \exp(- \f{\norm \bmu}\sigma \cos \theta_t)}{ 2C_\ell U'} \int_0^R \f {r}{\sigma} \dx r \\
    &= \f{ R^2 \exp(- \f{\norm \bmu}{\sigma} \cos \theta_t)}{ 6C_\ell  U'}.\label{eq:temperature.lowerbound.intermediate}
\end{align}
In $(i)$ we use that $\ell$ is decreasing and hence $-\ell'\geq 0$, as well as $\sin\phi \in [0,1]$ for $\phi\in [0,\pi/2]$.  In $(ii)$ we use Definition \ref{assumption:loss}.   In $(iii)$ we have used the assumption that $\theta_t \in [0,\pi/2]$ and that $\cos \theta_t \geq 0$ for $\theta_t\in [0,\pi/2]$.  In $(iv)$ we use that $\int_0^{\pi/2} \exp(-a \cos x) \sin x \dx x = (1 - \exp(-a))/a$.   In $(v)$ we have used the definition of anti-anti-concentration.   
In $(vi)$ we use that $\sigma \geq R$ and that $1-\exp(-x)\geq x/2$ on $[0,1]$.    Putting \eqref{eq:temperature.lowerbound.init} together with \eqref{eq:temperature.lowerbound.intermediate}, we get
\begin{equation}
\E[-\ell'(|\sip{\bbeta_t}{y\xb}|) \tilde \bmu^\top(y\xb)] \geq \f{ R^2 \norm \bmu \exp(-\f{ \norm \bmu}\sigma \cos \theta_t )}{6 C_\ell U'} \cdot \sin^2(\theta_t).
    \label{eq:temperature.lowerbound}
\end{equation}
We now want an upper bound on the second term in \eqref{eq:l'.decomp}.  Using the same coordinate system defined in terms of $\bbeta_t=e_2$ and $\bar \bmu=(\sin \theta_t, \cos \theta_t)$, we have that
\[ S_t^c = \{ \sip{\bbeta_t}{y\xb} < 0 \} = \{ \sip{\bbeta_t}{\zb + \bmu} < 0 \} = \{ u_2 + \norm \bmu \cdot \cos \theta_t < 0 \}.\]
Thus, 
\begin{align}\nonumber
&\E\l[  \ell'(|\sip{\bbeta_t}{y\xb}| / \sigma )  \tilde \bmu^\top y\xb \cdot  \ind(S_t^c)\r]\\  \nonumber 
&=\int_{-\infty}^\infty \int_{-\infty}^\infty \ell'(|u_2 + \|\bmu\|\cos \theta_t| / \sigma ) \cdot (u_1 \sin \theta_t +\|\bmu\|\sin^2 \theta_t )\ind(u_2 + \|\bmu\|\cos \theta_t \leq 0) \cdot p_{\bbeta_t, \bar \bmu}(u_1, u_2)  \dx u_1 \dx u_2\\ \nonumber 
&= \int_{-\infty}^\infty  \ell'(|u_2 + \|\bmu\|\cos \theta_t| / \sigma) \cdot \ind(u_2 + \|\bmu\|\cos \theta_t \leq 0) \bigg[\int_{-\infty}^\infty(u_1 \sin \theta_t +\|\bmu\|\sin^2 \theta_t )\cdot p_{\bbeta_t, \bar \bmu}(u_1, u_2)  \dx u_1 \bigg]\dx u_2 \\  \nonumber 
&\overset{(i)}=\norm \bmu \sin^2  \theta_t\int_{-\infty}^\infty  \ell'(|u_2 + \|\bmu\|\cos \theta_t| / \sigma )  \cdot \ind(u_2 + \norm \bmu \cos \theta_t \leq 0) \int_{u_1=-\infty}^\infty  p_{\bbeta_t, \bar \bmu}(u_1, u_2) \dx u_1  \dx u_2.
\end{align}
In $(i)$ we use the fact that $\zb$ isotropic implies that for all $u_{2}$ we have $\int_{-\infty}^\infty u_1\cdot p_{\bbeta_t, \bar \bmu}(u_1, u_2)  \dx u_1 = 0$.  Using $|\ell'(z)|\leq 1$ on $[0,\infty)$, we can therefore bound
\begin{align}
    \l| \E\l[  \ell'(|\sip{\bbeta_t}{y\xb}| / \sigma )  \tilde \bmu^\top y\xb \cdot  \ind(S_t^c)\r] \r|\leq \norm \bmu \sin^{2} \theta_t \cdot\P(S_t^c) =\norm \bmu \sin^{2} \theta_t  \cdot \err_t .\label{eq:l'.ub.temperature}
\end{align}
We then return to \eqref{eq:l'.decomp}. By combining \eqref{eq:l'.ub.temperature} with \eqref{eq:temperature.lowerbound}, we see that
\begin{align} \nonumber 
\sip{\bmu}{-\EE [\nabla \hat {\lunsupt t}(\bbeta_t)]} &=  \f{\norm \bmu}\sigma \E[ -\ell' \big(|\sip{\bbeta_t}{y\xb}| / \sigma \big)\tilde \bmu^\top (y\xb)]  + \f{ 2\norm \bmu}{\sigma} \E\l[\ell' \big(|\sip{\bbeta_t}{y\xb}| / \sigma \big)\tilde \bmu^\top (y\xb)\ind(S_t^c)\r]\\  \nonumber 
&\geq \f {\norm \bmu}{\sigma} \l[ \f{ R^2 \norm \bmu \exp(-\f{ \norm \bmu}\sigma \cos \theta_t )}{6 C_\ell U'} \cdot \sin^2(\theta_t) \r] - 2 \f{ \norm \bmu}{\sigma} \l[ \norm \bmu \sin^{2} \theta_t  \cdot \err_t \r] \\ \nonumber 
&= \f{ \norm \bmu^2 \cdot \sin^2 \theta_t }{\sigma} \l[ \f{ R^2 \exp(- \f{\norm \bmu}\sigma \cos \theta_t)}{6 C_\ell U'} - 2 \err_t\r] \\  \nonumber 
&\geq \f{ \norm \bmu^2 \cdot \sin^2 \theta_t }{\sigma} \l[ \f{ R^2  \exp(- \f{\norm \bmu}\sigma) }{6 C_\ell U'} - 2 \err_t\r].
\end{align}
Thus, by choosing $\sigma \geq \norm \bmu$, we have 
\begin{align} \nonumber 
\sip{\bmu}{-\EE [\nabla \hat {\lunsupt t}(\bbeta_t)]}&\geq  \f{ \norm \bmu^2 \cdot \sin^2 \theta_t }{\sigma} \l[ \f{ R^2  \exp(-1) }{6 C_\ell U'} - 2 \err_t\r] \\ \nonumber 
&\geq \f{ \norm \bmu^2 \cdot \sin^2 \theta_t }{\sigma} \l[ \f{ R^2 }{18 C_\ell U'} - 2 \err_t\r].
\end{align}
In particular, if 
\[ \err_t \leq \f {R^2}{72 C_\ell U'} =: \cerr, \]
then we have
\begin{equation} \nonumber 
    \sip{\bmu}{-\EE [\nabla \hat {\lunsupt t} (\bbeta_t)]} \geq \f{ \norm \bmu^2 R^2 }{36 \sigma C_\ell U'} \sin^2 \theta_t.
\end{equation}

\end{proof}

\subsection{Proof of Lemma \ref{lemma:innerproduct_grad_both}, batch of samples}
We now prove the second part of Lemma~\ref{lemma:innerproduct_grad_both}, where we show that provided the batch size is large enough, then the same lower bound that holds from the expected value holds when using finite samples.
\begin{lemma}[Lemma \ref{lemma:innerproduct_grad_both}, batch of samples]
Let $(\xb, y)\sim \calD$ be a mixture model with mean $\bmu\in \R^d$ and parameters $K, U, U', R>0$.   Let $\ell$ be well-behaved for some $C_\ell\geq 1$ and assume the temperature satisfies $\sigma \geq R \vee \norm \bmu$. Suppose that $\theta_t\in [0,\pi/2]$ is the angle between $\bmu$ and $\bbeta_t$ where $\norm {\bbeta_t}=1$.  Then there exists a universal constant $C_B>0$ such that for any $\eps, \delta\in (0,1)$,
\[ B \geq C_B \l( \f{ K C_\ell U'}{R^2}\r)^2 \eps^{-1} \log(2/\delta),\]
then with probability at least $1-\delta$,
\[ \sip{\bmu}{-\nabla \hat {\lunsupt t}(\bbeta_t)} \geq \f{ R^2 \norm{\bmu}^2}{72 \sigma C_\ell U'}  \sin^2\theta_t - \eps/2.\]
\end{lemma}
\begin{proof}[Proof of Lemma \ref{lemma:innerproduct_grad_both}, batch of samples]
Denote by $\zb_i = y_i\xb_i-\bmu$ and $S_t = \{y_i=\sgn(\sip{\bbeta_t}{\xb_i})\}$.  Let $\bar \bmu = \bmu/\norm \bmu$ and $\tilde \bmu = (I - \bbeta_t \bbeta_t^\top)\bar \bmu$.  
Let us define
\begin{align*}
A_t &= \frac{1}{B}\summ i B -\ell' \big(|\sip{\bbeta_t}{\xit it}| / \sigma \big)\cdot\sgn(\sip{\bbeta_t}{\xit it})\cdot\tilde \bmu ^\top \xit it,
\end{align*}
so that $\sip{\bmu}{-\E \nabla \hat {\lunsupt t}(\bbeta_t)} = \f{\norm \bmu}\sigma \E A_t$.  
We can use $\sip{\tilde \bmu}{\xit it} = \sip{\tilde \bmu}{\zb_i + y\bmu}$ to write the above as the sum of two terms:
\begin{align*}
A_t^{(1)} &= \frac{1}{B}\summ i B u_i^{(1)} := \f 1 B \summ i B  -\ell' \big(|\sip{\bbeta_t}{\xit it}|\big)\cdot\sgn(\sip{\bbeta_t}{\xit it})\cdot \sip{\tilde \bmu}{\zb_i},\\
A_t^{(2)} &= \frac{1}{B}\summ i B u_i^{(2)} := \f 1 B \summ i B -\ell' \big(|\sip{\bbeta_t}{\xit it}|\big)\cdot\sgn(\sip{\bbeta_t}{\xit it})\cdot \norm \bmu \cdot y\sip{\tilde \bmu}{\bar \bmu}.
\end{align*}
First note 
\begin{align*}
    \snorm{\tilde \bmu}^2 &= \snorm{\bar \bmu - \bbeta_t \sip{\bbeta_t}{\bar \bmu}}^2 \\
    &= \snorm{\bar \bmu - \bbeta_t \cos \theta_t}^2 \\
    &= 1 + \cos^2 \theta_t - 2 \cos^2 \theta_t \\
    &= \sin^2 \theta_t.
\end{align*}
Since $\zb_i$ are i.i.d. isotropic, each $u_i^{(1)}$ are i.i.d. mean zero random variables with sub-exponential norm at most $\norm {\tilde \bmu} \cdot \pnorm{\sip{ \norm {\tilde \bmu} / \norm{\tilde \bmu}}{\zb}} {\psi_2} \leq  \norm {\tilde \bmu} K=K \sin \theta_t$.   Thus, using sub-exponential concentration~\citep[Proposition 5.16]{vershynin}, there exists a universal constant $C>0$ such that for any $\xi>0$,
\begin{equation} \P(|A_t^{(1)} - \EE A_t^{(1)}| \geq \xi)  \leq 2 \exp\l(-C \min\l( \f{ \xi^2 B}{ K^2 \sin^2 \theta_t }, \f{ \xi B}{K \sin \theta_t} \r)\r). \nonumber 
\end{equation}
By taking $\xi = C K \sin \theta_t \sqrt{\f 1 B \log(2/\delta)}$ for a sufficiently large constant $C>0$, this implies that with probability at least $1-\delta/2$, 
\begin{equation} \label{eq:at1.bound}
|A_t^{(1)} - \EE A_t^{(1)}| \leq C K\sin \theta_t \sqrt{\f 1 B \log(2/\delta)}.
\end{equation}
On the other hand, each of $u_i^{(2)}$ are i.i.d. sub-exponential random variables with mean $\E A_t$ and sub-exponential norm at most $\norm \bmu \pnorm{\sip{\tilde \bmu}{\bar \bmu}}{\psi_2} = \norm \bmu(1- \cos \theta_t)\leq \norm \bmu \sin^2 \theta_t$ (using $\theta_t \in [0,\pi/2]$) and thus for any $\xi>0$,
\begin{equation} \P(|A_t^{(2)} - \E A_t^{(2)}| \geq \xi)  \leq 2 \exp\l(-C \min\l( \f{ \xi^2 B}{ \norm \bmu^2 \sin^4 \theta_t }, \f{ \xi B}{\norm \bmu \sin^2 \theta_t} \r)\r).  \nonumber 
\end{equation}
By taking $\xi = C \norm \bmu  \sin^2 \theta_t \sqrt{\f 1B \log(2/\delta)}$ for $C$ sufficiently large universal constant, we get that with probability at least $1-\delta/2$,
\begin{equation}\label{eq:at2.bound}
|A_t^{(2)} - \E A_t^{(2)}| \leq C \norm \bmu  \sin^2 \theta_t \sqrt{\f 1 B \log(2/\delta)}.
\end{equation}
Putting \eqref{eq:at1.bound} and \eqref{eq:at2.bound} together and applying union bound, we have that with probability at least $1-\delta$,
\begin{align}   \nonumber 
    \sip{\bmu}{- \nabla \hat {\lunsupt t}(\bbeta_t)} &= \f{\norm \bmu}{\sigma}\l[ A_t^{(1)} + A_t^{(2)} \r] \\  \nonumber 
    &\overset{(i)}\geq \f {\norm \bmu}{\sigma} \l[ \E A_t - C K \sin \theta_t \sqrt{\f 1 B \log(2/\delta)} - C \norm \bmu \sin^2 \theta_t \sqrt{ \f 1 B \log(2/\delta)} \r]\\  \label{eq:batch.gradient.init}
    &\overset{(ii)}\geq \f{\norm \bmu}{\sigma} \l[ \f{ \norm \bmu R^2}{36 C_\ell U'} \sin^2 \theta_t - C K \sin \theta_t \sqrt{\f 1 B \log(2/\delta)} - C \norm \bmu \sin^2 \theta_t \sqrt{ \f 1 B \log(2/\delta)} \r].
\end{align} 
In $(i)$ we use \eqref{eq:at1.bound} and~\eqref{eq:at2.bound}.   In $(ii)$ we use Lemma~\ref{lemma:innerproduct_grad_mean_expectation}.  Now, to complete the proof, we consider separately the case that $\sin^2 \theta_t > \eps$ and the case that $\sin^2 \theta_t \leq \eps$.  In the first instance, the batch size satisfies of
\[ B \geq  \l( \f{ 144 K C^2 C_\ell U'}{R^2} \r)^2 \eps^{-1} \log(2/\delta) \geq \l( \f{ 144 K C^2 C_\ell U'}{R^2 \sin \theta_t} \r)^2 \log(2/\delta).\]
Thus using $K \geq 1$ and $\norm \bmu \geq 1$, we have
\begin{align} \nonumber 
    CK \sin \theta_t \sqrt{\f 1 B \log(2/\delta)} &\leq \f {R^2}{144  C_\ell U'} \sin^2 \theta_t \leq \f {\norm \bmu R^2}{144 C_\ell U'} \sin^2 \theta_t,
\end{align}
and
\begin{align} \nonumber 
    C \norm \bmu \sin^2 \theta_t \sqrt{\f 1 B \log(2/\delta)} &\leq \f {\norm \bmu R^2}{144 K C_\ell U'} \sin^3 \theta_t \leq \f {\norm \bmu R^2}{144 C_\ell U'} \sin^2 \theta_t,
\end{align}
Substituting this into~\eqref{eq:batch.gradient.init}, we get
\begin{equation}\label{eq:batch.gradient.final} \sip{\bmu}{-\nabla \hat {\lunsupt t}(\bbeta_t)} \geq\f{\norm \bmu}{\sigma} \l[ \f{ \norm \bmu R^2}{72 C_\ell U'} \sin^2 \theta_t \r] \geq  \f{\norm \bmu}{\sigma} \l[ \f{ \norm \bmu R^2}{72 C_\ell U'} \sin^2 \theta_t - \eps/2. \r] 
\end{equation}
On the other hand, if $\sin^2 \theta_t \leq \eps$, then notice that~\eqref{eq:batch.gradient.init} becomes
\begin{align}
    \sip{\bmu}{-\nabla \hat {\lunsupt t}(\bbeta_t)} &\geq \f{\norm \bmu}{\sigma} \l[ \f{ \norm \bmu R^2}{36 C_\ell U'} \sin^2 \theta_t - C K \eps^{1/2} \sqrt{\f 1 B \log(2/\delta)} - C \norm \bmu \eps \sqrt{ \f 1 B \log(2/\delta)} \r].
\end{align}
Then $B = \Omega(\eps^{-1})$ implies that~\eqref{eq:batch.gradient.final} holds in this case as well.  This completes the proof. 
\end{proof}

\section{Proofs from Section \ref{sec:supervised}}\label{appendix:supervised}
In this section, we prove the following theorem. 
\begin{theorem}[Theorem \ref{thm:supervised.logistic.regression}, restated]
Let $(\xb, y)\sim \calD$ be a mixture distribution with mean $\mu$ and parameters $K, U, U', R>0$.  Let $\cerr>0$ be arbitrary, and assume $\norm \bmu \geq 3K \max(\log(8/\cerr), 22K)$.   By running Algorithm~\ref{alg:logistic} with $\eta = (\norm{\bmu}^2 + d)^{-1} \cerr /8$ and $T = 8 \eta^{-1} \cerr^{-1} \norm{\bmu}^2$ iterations, there exists $i\leq 4 \log(1/\delta)$ and $t<T$ such that with probability at least $1-\delta$, 
\[ \P(y \neq \sgn(\sip{\bti ti}{\xb})) \leq \cerr.\]
\end{theorem}
To show this theorem, we will first need an upper bound for the classification error that is achieved by the classifier $\xb\mapsto \sgn(\sip {\bar \bmu} \xb)$.  For the standard isotropic Gaussian mixture model $N(y \bmu, I_d)$, it is easy to show that $\P(y \neq \sgn(\sip {\bar \bmu}{\xb})) = \Phi(-\norm \bmu)$, where $\Phi$ is the standard normal CDF.  For sub-exponential mixture models, we have a similar bound.
\begin{lemma}\label{lemma:subexp.mixture.error.ub}
Let $(\xb, y)\sim \calD$ be a mixture model with mean $\bmu$ and parameters $K, U, U', R>0$.   Then we have,
\begin{equation} \nonumber 
    \P\l(y \neq \sgn(\sip{\xb}{\bmu})\r) \leq K \exp(-\norm{\bmu}/K).
\end{equation}
\end{lemma}
\begin{proof}
For simplicity denote $\bar \bmu = \bmu/\norm{\bmu}$.  We have
\begin{align*}
    \P\l(y \neq \sgn(\sip{\xb}{\bmu})\r) &= \P( \sip{y\xb}{\bar \bmu } < 0) \\
    &= \P(\sip{y\xb-\bmu}{\bar \bmu} < - \norm \bmu) \\
    &= \P(\sip{\zb}{\bar \bmu} < - \norm{\bmu}) \\
    &= \int_{-\infty}^{-\norm{\bmu}} \P(\sip{\zb}{\bar \bmu} < -t ) \dx t \\
    &\leq \int_{-\infty}^{-\norm{\bmu}} \exp(-|t|/K) \dx t\\
    &= K \exp(-\norm{\bmu}/K).
\end{align*}
The inequality uses the definition of sub-exponential. 
\end{proof}

The next intermediate result we need will be a characterization of the population loss under a surrogate for the 0-1 loss.  
\begin{lemma}\label{lemma:population.risk.exp.-mu}
Let $\ell$ be 1-Lipschitz, decreasing, with $\ell(z)\leq \exp(-z)$ for $z>0$.  Let $(\xb, y)\sim \calD$ be a mixture model with mean $\bmu$ and parameters $K, U, U', R>0$.   Then
\[ \E_{(x,y)\sim \calD} \ell(y\sip{\bmu}{x}) \leq (1 + \norm \bmu + 2 \norm{\bmu}^2) K \exp(-\norm \bmu/K) + \exp(-\norm \bmu / 2K) + \exp(- \norm \bmu / 2).\]
In particular, provided $\norm{\bmu}\geq 64K^2$, we have
\[ \E_{(\xb,y)\sim \calD} \ell(y \sip \bmu \xb) \leq \exp(-\norm \bmu/ 3K).\]
\end{lemma}
\begin{proof}
Denote $\bar \bmu = \bmu/\norm{\bmu}$ for simplicity, and let $\gamma = 1/2$.   Our proof uses an argument similar to~\citet[Lemma 5.9]{frei2020halfspace} and~\citet[Lemma 2.7]{zoufrei2021adversarial}, where we decompose 
\begin{align}\nonumber 
    \E\ell(\sip{y\xb}{\bmu}) &= \E[\ell(\sip{y\xb}{\bmu}) \ind(\sip{y\xb}{\bar \bmu} < 0) ]\\ \nonumber
    &\quad + \E[\ell(\sip{y\xb}{\bmu}) \ind(\sip{y\xb}{\bar \bmu} \in [0,\gamma) )] \\
    &\quad + \E[\ell(\sip{y\xb}{\bmu}) \ind(\sip{y\xb}{\bar \bmu} > \gamma )]. \label{eq:surrogate.vs.zeroone.decomposition}
\end{align}
We first bound the first term.  Denote by $\opt$ the classification error using $\bmu/\norm \bmu$,
\[ \opt = \P(y\neq \sgn(\sip{\xb}{\bmu / \norm \bmu)}) \leq K \exp(-\norm \bmu / K),\]
where the inequality follows by Lemma \ref{lemma:subexp.mixture.error.ub}.  Let $\xi = 2 \norm{\bmu}$.  We have
\begin{align}  \nonumber 
    \E[\ell(\sip{y\xb}{\bmu}) \ind(\sip{y\xb}{\bar \bmu} < 0)] &\overset{(i)}\leq \E[(1 + |\sip{y\xb}{\bmu}|) \ind(\sip{y\xb}{\bar \bmu} < 0) ] \\ \nonumber 
    &= \opt + \norm{\bmu} \E[|\sip{y\xb}{\bar \bmu}| \ind(\sip{y\xb}{\bar \bmu}<0,\ |\sip{y\xb}{\bar \bmu}| \leq \xi) \\ \nonumber  \nonumber  \nonumber 
    &\quad + \norm{\bmu} \E[|\sip{y\xb}{\bar \bmu}| \ind(\sip{y\xb}{\bar \bmu}<0,\ |\sip{y\xb}{\bar \bmu}| > \xi)\\ \nonumber  \nonumber 
   &\leq (1 + \norm{\bmu} \xi) \opt + \norm{\bmu} \E[|\sip{y\xb}{\bar \bmu}| \ind(\sip{y\xb}{\bar \bmu} > \xi)]\\ \nonumber 
    &= (1 + \norm{\bmu} \xi ) \opt + \norm\bmu \int_{\xi }^{\infty} \P(\sip{y\xb}{\bar \bmu} > t) \dx t \\ \nonumber
    &= (1 + \norm{\bmu} \xi ) \opt + \norm\bmu \int_{\xi }^{\infty} \P(\sip{y\xb - \bmu}{\bar \bmu} > t - \bmu) \dx t \\ \nonumber 
    &\overset{(ii)}\leq (1 + \norm{\bmu} \xi ) \opt + \norm\bmu \int_{\xi }^{\infty} \exp(-(t-\norm{\bmu})/K)  \dx t \\ \nonumber 
    &=  (1 + \norm{\bmu} \xi ) \opt + K \norm\bmu \exp((\norm{\bmu}-\xi )/K)\\ \label{eq:convex.surrogate.01.firstterm}
    &= (1 + 2 \norm{\bmu}^2 ) \opt + K \norm{\bmu}\exp(-\norm{\bmu} / K).
\end{align}
In $(i)$ we use Cauchy--Schwarz and that $\ell$ is 1-Lipschitz and decreasing.  In $(ii)$ we use that $t \geq \xi \geq \norm{\bmu}$ and the definition of sub-exponential.

For the second term of \eqref{eq:surrogate.vs.zeroone.decomposition}, we have
\begin{align} \nonumber 
    \E[\ell(\sip{y\xb}{\bmu}) \ind(\sip{y\xb}{\bar \bmu} \in [0,\gamma) )] &\leq \ell(0) \P(\sip{y\xb}{\bar \bmu} \in [0,\gamma)) \\ \nonumber 
    &\leq \P(\sip{y\xb - \bmu}{\bar \bmu} \in [-\norm\bmu, -\norm \bmu + \gamma) ) \\ \nonumber 
    &= \P(\sip{y\xb-\bmu}{\bar \bmu} \leq -\norm \bmu + \gamma) - \P(\sip{y\xb - \bmu}{\bar \bmu} \leq - \norm \bmu) \\ \nonumber 
    &\overset{(i)}\leq \P((\sip{y\xb-\bmu}{\bar \bmu} \leq -\f 12 \norm{\bmu})\\
    \label{eq:convex.surrogate.01.secondterm}
    &\overset{(ii)}\leq \exp(-\norm{\bmu}/2 K).
\end{align}
where in $(i)$ we use $\gamma < 1 \leq \f 1 2 \norm{\bmu}$ and in $(ii)$ we have used the definition of sub-exponential.  

Finally, for the third term of \eqref{eq:surrogate.vs.zeroone.decomposition}, we use that $\ell$ is decreasing and has exponential tail so that
\begin{equation} \label{eq:convex.surrogate.01.thirdterm}
\E [\ell(\sip{y\xb}{\bmu}) \ind(\sip{y\xb}{\bar \bmu} > \gamma)] \leq \ell(\norm{\bmu}) \leq \exp(-\gamma \norm{\bmu}) = \exp(-\norm{\bmu}/2).
\end{equation} 
Putting \eqref{eq:convex.surrogate.01.firstterm}, \eqref{eq:convex.surrogate.01.secondterm}, and \eqref{eq:convex.surrogate.01.thirdterm} all together, we get
\begin{align} \nonumber 
    \E[\ell(y \sip{\bmu}{x})] &\leq (1 + K \norm{\bmu}^3 ) \opt + K \norm{\bmu}\exp(-\norm{\bmu}^2/2) + \exp(-\norm{\bmu}/(2K)) + \exp(-\norm{\bmu}/2) \\ \nonumber 
    &\leq (1 + 2 \norm{\bmu}^2) K \exp(-\norm{\bmu} / K) + K \norm{\bmu}\exp(-\norm{\bmu}^2/2) + 2\exp(-\norm{\bmu}/(2K)) \\ \nonumber 
    &\overset{(i)}\leq (3K+2) \exp(-\norm \bmu/2K) \\ \nonumber 
    &\overset{(ii)}\leq  \exp(-\norm \bmu/3K).
\end{align}
In $(i)$ we use that $x/\log x \geq \sqrt x$ and thus $2\norm \bmu^2 \exp(-\norm \bmu / K) = \exp(-\f 1 k \norm \bmu + 4 \log \norm \bmu) \leq \exp(-\norm \bmu / 2K)$ for $\norm \bmu \geq 64K^2$, and in $(ii)$ we again use that $\norm \bmu \geq 64K^2$.  
\end{proof}

With the above in hand we can complete the proof of Theorem~\ref{thm:supervised.logistic.regression}.  We will show that provided the means of the mixture model are sufficiently well-separated (by an absolute constant), then the population risk under the convex surrogate $\ell$ can be as small as $\Theta(\cerr)$.  This leads to an upper bound for the classification error using supervised learning that is at most $\cerr$.  
\begin{proof}[Proof of Theorem \ref{thm:supervised.logistic.regression}]
Fix $i\in \{1,\dots, \log(1/\delta)\}$ as given in Algorithm \ref{alg:logistic}.  As the cross-entropy loss is convex and 1-Lipschitz, and as $\E[\norm{\xb}^2] \leq 2 \norm{\bmu}^2 + 2 \E[\norm{\zb}^2] = 2(\norm \bmu^2 + d)$, by~\citet[Lemma C.1]{frei2020halfspace}, for $\eta \leq (\norm{\bmu}^2+d)^{-1} \eps/4$, we know there exists $t_i<T = 4\eta^{-1} \eps^{-1} \norm{\bmu}^2$ such that $\E[\ell(y\sip{\bbeta{t_i}}{\xb}] \leq \E [\ell(y \sip{\bmu}{\xb})] + \eps/2$.  By Markov's inequality, for each $i$, with probability at least $1-\f{1}{1+\delta_0}$ over $\{(\xit ti, y_t^{(i)})\}_{t=0, \dots, T}$, $\E[\ell(y\sip{\bbeta_{t_i}}\xb)]\leq (1 + \delta_0) [\E \ell(y \sip{\bmu}{\xb}) + \eps]$.  As the $\{(\xit ti, y_t^{(i)})\}$ are independent for different $i$, the probability of failure for $I$ independent such $i$ is $[1/(1+\delta_0)]^I$.  As $1/x \leq 1 / \log(1+x)\leq 2/x$ on $[0,1]$, this implies that as long as $I \geq 2 \delta_0^{-1} \log(1/\delta)$, then with probability at least $1-\delta$, there exists $i\in I$ such that $\E[\ell(y\sip{\bbeta_{t_i}}\xb)]\leq (1 + \delta_0) [\E \ell(y \sip{\bmu}{\xb}) + \eps]$.  In particular, for $I = 4 \ceil{\log(1/\delta)}$, we have with probability at least $1-\delta$, for some $i$ and $t_i<T$,
\begin{equation} \nonumber 
    \E \ell(y\sip{\bbeta_{t_i}}\xb) \leq 2 \E \ell(y \sip{\bmu}{\xb}) + \eps.
\end{equation}
By Lemma \ref{lemma:population.risk.exp.-mu}, we know that for $\norm{\bmu}\geq 64K^2$, we have
\[ \E \ell(y\sip{\bmu}{\xb}) \leq \exp(-\norm \bmu/3K).\]
To guarantee $2\exp(-\norm{\bmu}/3K) \leq \cerr \log(2) / 2$ it suffices to take $\norm \bmu \geq 3K \log(8 / \cerr)$.  Thus, provided $\norm \bmu \geq 3K \max(\log(8/\cerr), 22K)$, we have that with probability at least $1-\delta$,
\[ \P(y \neq \sgn(\sip{\bbeta_{t_i}}{\xb})) \leq \f 1 {\ell(0)} \E \ell(y \sip{\bbeta_{t_i}}\xb) \leq \f 1 2 \cerr + \eps.\]
Taking $\eps = \cerr/2$ completes the proof.
\end{proof}

\section{Bayes-optimal Classifier for Mixture Distributions}\label{appendix:bayes.opt}
Here we prove a more general version of Fact~\ref{fact:bayes.opt} that relies only upon rotational symmetry and unimodality. 
\begin{fact}\label{fact:bayes.opt.appendix}
Let $\bmu \in \R^d$.  Suppose $\zb$ is continuous, rotationally symmetric and unimodal in the sense that its density function $p_{\zb}(\zb)$ is a decreasing function of $\norm{\zb}_2$.  Assume $Y\sim \Unif(\{1,-1\})$ and $\xb|Y=y \sim \zb + y \bmu$.  Then the Bayes-optimal classifier is given by $\xb \mapsto \sgn(\sip{\bmu}{\xb})$. 
\end{fact}
\begin{proof}
Let us introduce some notation.  We denote by $\boldsymbol{X}, \boldsymbol{Z}, Y$ as random variables and by $\zb, \mathrm  \xb\in \R^d$ and $y \in \{-1,1\}$ as possible realizations of those random variables.  The Bayes-optimal classifier chooses a label for a feature $\xb \in \R^d$ by taking the maximum value of $\mathbb{P}(Y= y| \xb)$ over $y\in \{\pm 1\}$.   Thus, we can write the Bayes-optimal classifier $h_{\mathrm{Bayes}}(\mathbf x)$ as
\begin{equation*}
    h_{\mathrm{Bayes}} (\mathbf x)= \mathrm{argmax}_{y\in \{\pm 1\}} \mathbb{P}(Y=y |\mathbf  x).
\end{equation*}
Note that $\mathbb{P}(Y=1)=\mathbb{P}(Y=-1)=1/2$. 
Thus, by Bayes' theorem,
\begin{align*}
    \mathbb{P}(Y=y |\mathbf  x) &= \frac{ \mathbb{P}(\mathbf x | Y= y) \mathbb{P}(Y=y) }{\mathbb{P}(\mathbf x | Y=1) \mathbb{P}(Y=1) + \mathbb{P}(\mathbf x|Y=-1) \mathbb{P}(Y=-1)}\\
    &= \frac{ \mathbb{P}(\mathbf x | Y=y)}{\mathbb{P}(\mathbf x|Y=1) + \mathbb{P}(\mathbf x | Y=-1)}.
\end{align*}
Thus, we see that the Bayes-optimal classifier chooses the label for a feature $\mathbf x$ which maximizes the likelihood of observing $\mathbf x$:
\begin{equation*}
    h_{\mathrm{Bayes}}(\mathbf x) = \mathrm{argmax}_{y\in \{\pm 1\}} \mathbb{P}(\mathbf x | Y=y).
\end{equation*}
If we denote by $p_\zb(\cdot)$ as the probability density function of $\zb$, since $\xb|Y=y \sim \zb + y\bmu$, using the properties of the probability density function under linear transformations, we have
\[ \mathbb{P}(\mathbf x | Y=y) = p_{\zb}(\mathbf x - y \bmu).\]
Thus,
\begin{equation}\label{eq:bayes.opt.decision}
    h_{\mathrm{Bayes}}(\mathbf x) = \begin{cases}
    1 &\text{if }p_{\zb}(\mathbf x-\bmu) > p_{\zb}(\mathbf x+ \bmu),\\
    -1 &\text{if } p_{\zb}(\mathbf x-\bmu) \leq p_{\zb}(\mathbf x+ \bmu).
    \end{cases}
\end{equation}

By assumption, there exists a decreasing function $g:[0,\infty) \to [0,\infty)$ such that the density function of $\zb$ satisfies
\[ p_{\zb}(\mathbf  z) = g(\norm {\mathbf z}_2^2).\]
By~\eqref{eq:bayes.opt.decision}, the Bayes-optimal decision boundary is determined by comparing $p_{zb}(\mathbf x- \bmu)$ and $p_{\zb}(\mathbf x+\bmu)$.  We have,
\begin{align*}
    p_{\zb}(\mathbf x-\bmu) > p_{\zb}(\mathbf x+\bmu) &\overset{(i)}\iff g(\snorm{\mathbf x-\bmu}_2^2) > g(\snorm{\mathbf x+\bmu}_2^2)\\
    &\overset{(ii)}\iff \snorm{\mathbf x-\bmu}_2^2 \leq \snorm{\mathbf x+\bmu}^2 \\
    &\iff \snorm {\mathbf x}_2^2 + \snorm {\bmu}_2^2 - 2 \sip {\mathbf x}\bmu \leq  \snorm {\mathbf x}_2^2 + \snorm {\bmu}_2^2 + 2 \sip{\mathbf x}{\bmu} \\
    &\iff \sip{\mathbf x}{\bmu} \geq 0.
\end{align*}
Above, $(i)$ follows by using rotational symmetry of $\zb$, and $(ii)$ follows by using that $\zb$ is unimodal so $g$ is decreasing. 

\end{proof}

\section{Remaining Proofs}\label{sec:concentration}
\begin{lemma}\label{lemma:concentration.subexp}
If $(\xb, y)\sim \calD$ is from a $K$-sub-exponential mixture model with mean $\bmu$, then for any $\delta>0$, with probability at least $1-\delta$, for any $i\in [B]$ and $t\in [T]$,
\[ \snorm{\xit it}^2 \leq 2 \norm{\bmu}^2 + 2 d K^2 \log^2(dBT/\delta).\]
\end{lemma}
\begin{proof} 
  Since the $\zb_i^{(t)} $ are $K$-subexponential, we have that for each component $j\in [d]$, for any $\xi >0$,
\[ \P( [\zb_i^{(t)} ]_j^2 \geq \xi) \leq \exp(-\sqrt \xi / K).\]
Since we have the inclusion for $\rho>0$,
\[ \{ \snorm{\zb_i^{(t)} }^2 \geq \rho \} \subset \cup_{j=1}^d \{ [\zb_i]_j^2 > \rho/d \},\]
we have that for any $i$,
\[ \P(\snorm{\zb_i^{(t)} }^2 \geq \rho) \leq d \P([\zb_i^{(t)}]_j^2 \geq \rho/d) \leq d \exp\l( -\f{ \sqrt{\rho}}{K \sqrt d}\r),\]
where we have used the fact that $\zb$ is $K$-sub-exponential.  
By taking $\rho = d K^2 \log^2 (dB/\delta)$, we get that with probability at least $1-\delta$, for any $i\in [B]$ and fixed $t$,
\begin{equation} \nonumber 
    \snorm{\zb_i^{(t)} }^2 \leq d K^2 \log^2(dB/\delta).
\end{equation}
Using Young's inequality, this implies
\[ \snorm{\xit it}^2 \leq 2 \norm{\bmu}^2 + 2 d K^2 \log^2(dB/\delta).\]
Scaling $\delta\mapsto \delta/T$ and using a union bound completes the proof.
\end{proof}

\begin{lemma}\label{lemma:recursion.calculation}
Suppose that we have the recursion
\[ \Delta_t^2 \leq (1 - \eta/2 C_g) \Delta_{t-1}^2 + \f{\eta \eps}{8 C_g} + \f{ 2 C_d \eta^2}{\sigma^2},\text{for } t = 1,\ldots,T,\]
where $C_d C_g^2 \sigma^2 \geq 1$ and $\Delta_0\leq 2$.\footnote{Note that $C_dC_g^2 \sigma^2 \geq 1$ for $C_d,C_g$ as in the proof of Theorem~\ref{thm:selftrain.unsupervised}, and that $\Delta_0= \norm{\beta_0-\bar \mu} \leq \norm{\beta_0} + \norm{\bar \mu} = 2$.}  
Then, for $\eta = \eps/(16 C_d C_g \sigma^2)$ and $T \geq 32 C_d C_g^2 \sigma^2 \eps^{-1} \log(32 C_d C_g^2 \sigma^2 \eps^{-1})$, we have $\Delta_T^2 \leq \eps$.
\end{lemma} 
\begin{proof} 
We unroll the recursion and use the geometric series formula to get
\begin{align*}
    \Delta_T^2 &\leq  \left(1 - \frac{ \eta}{2 C_g}\right)^T \Delta_0^2 + \l( \f{ \eta \eps}{8C_g} + \f{ 2C_d \eta^2}{\sigma^2} \r) \summm i 0 {T-1} \left(1 - \frac{ \eta}{2 C_g}\right)^i \\
    &= \left(1 - \frac{ \eta}{2 C_g}\right)^T \Delta_0^2 + \l( \f{ \eta \eps}{8C_g} + \f{ 2C_d \eta^2}{\sigma^2} \r) \frac{1 -  \left(1 - \frac{ \eta}{2 C_g}\right)^T}{\eta / 2 C_g} \\
    &\leq  \left(1 - \frac{ \eta}{2 C_g}\right)^T \Delta_0^2 + \eps/4 + 4 C_d C_g \eta/\sigma^2.
\end{align*}
Substituting the value for $\eta =  \varepsilon / (16 C_d C_g\sigma^2)$, we get
\begin{align*}
    \Delta_T^2 \leq \left( 1 - \frac{ \varepsilon}{32 C_d C_g^2 \sigma^2}\right)^T \Delta_0^2 + \eps/2.
\end{align*}
Thus, for $T \geq 32 C_d C_g^2 \sigma^2 \varepsilon^{-1} \log(32C_d C_g^2 \sigma^2 \varepsilon^{-1})$ and using the identity $(1-x)^{x^{-1} \log(1/x)} \leq x$ for $x\in (0,1)$, we get that (using $\Delta_0 \leq 2$)
\[ \Delta_T^2 \leq \f{\Delta_0^2 \eps}{32 C_d C_g^2 \sigma^2} + \eps/2 \leq \eps.\]
\end{proof}


\bibliographystyle{ims}
\bibliography{references}

\end{document}